\documentclass{article}
\usepackage{amssymb}
\usepackage{amsmath}
\usepackage{amsthm}
\usepackage{amstext} 
\usepackage{array}   
\usepackage{stmaryrd} 
\usepackage{color}
\usepackage{tikz}
\usepackage[all]{xy}
\usepackage{ijcai17}

\usepackage{times}

\newtheorem{theorem}{Theorem}[section]
\newtheorem{lemma}[theorem]{Lemma}
\newtheorem{proposition}[theorem]{Proposition}

\newtheorem{definition}[theorem]{Definition}
\newtheorem{example}[theorem]{Example}

\newtheorem{claim}{Claim}[theorem]
\newenvironment{claimproof}[1]{\par\noindent\textit{Proof of claim \theclaim:}\space#1}{\hfill $\blacksquare$}

\newcommand{\Kh}{{\ensuremath{\mathcal{K}h}}}
\newcommand{\K}{{\ensuremath{\mathcal{K}}}}
\newcommand{\PROP}{\textmd{PROP}}
\newcommand{\ACT}{\textmd{ACT}}

\newcommand{\lr}[1]{\langle #1 \rangle}
\newcommand{\lra}{\leftrightarrow}

\newcommand{\rel}[1]{\xrightarrow{#1}}

\newcommand{\calM}{\mathcal{M}}

\newcommand{\M}{\calM}

\newcommand{\calV}{\mathcal{V}}
\newcommand{\V}{\calV}

\newcommand{\llrr}[1]{\llbracket #1\rrbracket}
\newcommand{\EQ}[1]{{\ensuremath{[W]}}}
\newcommand{\LeafNode}{\textsf{leaf-node}}

\newcommand{\InnerNode}{\textsf{inner-node}}

\newcolumntype{L}{>{$}l<{$}} 

\newcommand{\TAUT}{\ensuremath{\mathtt{TAUT}}}
\newcommand{\NECK}{\ensuremath{\mathtt{NECK}}}

\newcommand{\DISTK}{\ensuremath{\mathtt{DISTK}}}
\newcommand{\AxTrK}{\ensuremath{\mathtt{T}}}
\newcommand{\AxTransK}{\ensuremath{\mathtt{4}}}
\newcommand{\AxEucK}{\ensuremath{\mathtt{5}}}
\newcommand{\SUB}{\ensuremath{\mathtt{SUB}}}
\newcommand{\MP}{\ensuremath{\mathtt{MP}}}
\newcommand{\EQREPKh}{\ensuremath{\mathtt{MONOKh}}}

\newcommand{\AxKtoKh}{\ensuremath{\mathtt{AxKtoKh}}}
\newcommand{\AxKhtoKhK}{\ensuremath{\mathtt{AxKhtoKhK}}}
\newcommand{\AxKhtoKKh}{\ensuremath{\mathtt{AxKhtoKKh}}}

\newcommand{\AxKhKh}{\ensuremath{\mathtt{AxKhKh}}}
\newcommand{\AxKhbot} {\ensuremath{\mathtt{AxKhbot}}}

\newcommand{\SKh}{\mathbb{SKH}}

\newcommand{\CELeafN}{\ensuremath{\mathtt{CELeaf}}}
\newcommand{\CEInnerN}{\ensuremath{\mathtt{CEInner}}}
\newcommand{\Dom}{\ensuremath{\mathtt{dom}}}

\newcommand{\cl}[1]{cl(#1)}
\newcommand{\SFC}{\Phi}

\title{Strategically Knowing How}
\author{Raul Fervari \\  Universidad Nacional de \\
    C\'ordoba, and CONICET,\\ Argentina \And Andreas Herzig \\  University of Toulouse,\\ IRIT, Toulouse, France\And Yanjun Li \\  University of Groningen\\ The Netherlands \And Yanjing Wang\thanks{corresponding author, \texttt{y.wang@pku.edu.cn}} \\  Peking University, \\ China}
\begin{document}
\maketitle
\begin{abstract}
In this paper, we propose a single-agent logic of goal-directed knowing how extending the standard epistemic logic of knowing that with a new knowing how operator. The semantics of the new operator is based on the idea that knowing how to achieve $\phi$ means that there exists a (uniform) strategy such that the agent knows that it can make sure $\phi$. We give an intuitive axiomatization of our logic and prove the soundness, completeness and decidability of the logic. The crucial axioms relating knowing that and knowing how illustrate our understanding of knowing how in this setting. This logic can be used in representing both knowledge-that and knowledge-how. 
\end{abstract}


\section{Introduction}

Standard epistemic logic focuses on reasoning about propositional knowledge expressed by \textit{knowing that} $\phi$ \cite{Hintikka:kab}. However, in natural language, various other knowledge expressions are also frequently used, such as \textit{knowing what}, \textit{knowing how}, \textit{knowing why}, and so on. 

In particular, \textit{knowing how} receives much attention in both philosophy and AI. Epistemologists debate about whether knowledge-how is also propositional knowledge \cite{Fantl08}, e.g., whether \textit{knowing how to swim} can be rephrased using \textit{knowing that}. In AI, it is crucial to let autonomous agents \textit{know how} to fulfill certain goals in robotics, game playing, decision making, and multi-agent systems. In fact, a large body of AI planning can be viewed as finding algorithms to let the autonomous planner \textit{know how} to achieve some propositional goals, i.e., to obtain goal-directed knowledge-how \cite{Gochet13}. Here, both propositional knowledge and knowledge-how matter, especially in the planning problems where initial uncertainty and non-deterministic actions are present. From a logician's point of view, it is interesting to see how \textit{knowing how} interact with \textit{knowing that}, and how they differ in their reasoning patterns. A logic of knowing how also helps us to find a consistency notion regarding knowledge database with knowing how expressions. 
\begin{example}\label{ex.pain}
Consider the scenario where a doctor needs a plan to treat a patient and cure his pain ($p$), under the uncertainty about some possible allergy ($q$). If there is no allergy ($\neg q$) then simply taking some pills can cure the pain, and the surgery is not an option. On the other hand, in presence of the allergy, the pills may cure the pain or have no effect at all, while the surgery can cure the pain for sure. The model from Figure~\ref{fig.curepain} represents this scenario with an additional action of testing whether $q$. The dotted line represents the initial uncertainty about $q$, and the test on $q$ can eliminate this uncertainty (there is no dotted line between $s_3$ and $s_4$). 
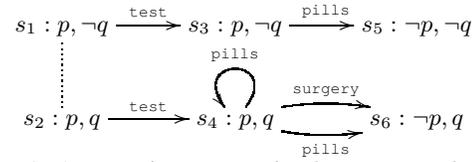
\begin{figure}
{\small \[\xymatrix{
s_1:p,\neg q\ar[r]^{\texttt{test}}\ar@{.}[d] & s_3:p,\neg q\ar[r]^{\texttt{pills}} & s_5:\neg p,\neg q\\
s_2:p,q\ar[r]^{\texttt{test}} & s_4:p,q \ar@/^/[r]^{\texttt{surgery}}\ar@(ur,ul)_{\texttt{pills}}\ar@/_/[r]_{\texttt{pills}}& s_6:\neg p, q 
}\]
} 
\vspace{-20pt}
\caption{A scenario representing how to cure the pain.}
\label{fig.curepain}
\end{figure}
According to the model, to cure the pain (guarantee $\neg p$) at the end, it makes sense to take the surgery if the result of the test of whether $q$ is positive and take the pills otherwise. We can say the doctor in this case \textit{knows how} to cure the pain. 

\end{example}

How to formalize the knowledge-how of the agent in such scenarios with uncertainty? Already since the early days of AI, people have been looking at it in the setting of logics of knowledge and action \cite{Mccarthy69,McCarthy79,Moore85,Les00,Hoek00}. However, there has been no consensus on how to capture the logic of  ``knowing how'' formally (cf. the recent surveys \cite{Gochet13} and \cite{KandA15}). The difficulties are well discussed in \cite{JamrogaA07} and \cite{Herzig15} and simply combining the existing modalities for ``knowing that'' and ``ability'' in a logical language like ATEL \cite{EATL} does not lead to a genuine notion of ``knowing how'', e.g., knowing how to achieve $p$ is not equivalent to knowing that there exists a strategy to make sure $p$. It does not work even when we replace the strategy by \textit{uniform strategy} where the agent has to choose the same action on indistinguishable states \cite{JamrogaA07}. Let $\phi(x)$ express that $x$ is a way to make sure some goal is achieved, and let $\K$ be the standard knowledge-that modality. There is a crucial distinction between the \textit{de dicto} reading of knowing how ($\K\exists x \phi(x)$) and the desired \textit{de re} reading ($\exists x \K \phi(x)$) endorsed also by linguists and philosophers \cite{stanley2001knowing}. The latter implies the former, but not the other way round. For example, consider a variant of Example 1.1 where no test is available: then the doctor has \textit{de dicto} knowledge-how to cure, but not the \textit{de re} one.  Proposals to capture the \textit{de re} reading have been discussed in the literature,  such as making the knowledge operator more constructive \cite{JamrogaA07}, making the strategy explicitly specified \cite{HerzigLW13,Belardinelli14}, or inserting $\K$ in-between an existential quantifier and the ability modality in see-to-that-it (STIT) logic \cite{Broersen2015}.

In \cite{Wang15lori,Wang2016}, a new approach is proposed by introducing a single new modality $\Kh$ of (conditional) goal-directed knowing how, instead of breaking it down into other modalities. This approach is in line with other \textit{de re} treatments of non-standard epistemic logics of knowing whether, knowing what and so on (cf.\ \cite{Wang16} for a survey). The semantics of $\Kh$ is inspired by the idea of conformant planning based on linear plans \cite{SW98,YLW15}. It is shown that $\Kh$ is not a normal modality, e.g, knowing how to get drunk and knowing how to drive does not entail knowing how to drive when drunk. The work is generalized further in \cite{LiWang17,Li17}. However, in these previous works, there was no explicit knowing that modality $\K$ in the language and the semantics of $\Kh$ is based on linear plans, which does not capture the broader notion allowing branching plans or strategies that are essential in the scenarios like Example \ref{ex.pain}. 

In this paper, we extend this line of work largely in the following aspects: 
\begin{itemize}
\item Both the \textit{knowing how} modality $\Kh$ and \textit{knowing that} modality $\K$ are in the language. 
\item In contrast to the state-independent semantics \cite{Wang15lori,Wang2016}, we interpret $\Kh$  \textit{locally} w.r.t.\ the current uncertainty.
\item Instead of linear plans in \cite{Wang15lori}, the semantics of our $\Kh$ operator is based on strategies (branching plans).
\end{itemize}
The intuitive idea behind our semantics of $\Kh$ is that the agent knows how to achieve $\phi$ iff (s)he has an executable uniform strategy $\sigma$ such that the agent knows that:  
\begin{itemize}
\item $\sigma$ guarantees $\phi$ in the end given the uncertainty; 
\item $\sigma$ always terminates after finitely many steps.
\end{itemize}
Note that for an agent to know how to make sure $\phi$, it is not enough to find a plan which works \textit{de facto}, but the agent should know it works in the end. This is a strong requirement inspired by planning under uncertainty, where the collection of final possible outcomes after executing the plan is required to be a subset of the collection of the goal states \cite{geffner2013concise}. 

Technically, our contributions are summarized as follows: 
\begin{itemize}
\item A logical language with both $\Kh$ and $\K$ operators with a semantics which fleshes out formally the above intuitions about knowing how.
\item A complete axiomatization with intuitive axioms.
\item Decidability of our logic. 
\end{itemize}
The paper is organized as follows: Section~\ref{sec.ls} lays out the language and semantics of our framework; Section~\ref{sec.ax} proposes the axiomatization and proves its soundness; We prove the completeness of our proof system and show the decidability of the logic in Section~\ref{sec.proof} before we conclude with future work. 

\section{Language and Semantics}\label{sec.ls}
Let \PROP\ be a countable set of propositional symbols.
\begin{definition}[Language]The language is defined by the following BNF where $p\in \PROP$:
	$$\phi:= p\mid \neg\phi\mid (\phi\land\phi)\mid \K\phi\mid \Kh{\phi}. $$
We use $\bot, \lor, \to$ as usual abbreviations and write $\hat{\K}$ for $\neg\K\neg$.  
\end{definition}
\begin{definition}[Models]
A model $\M$ is a quintuple $\lr{W,\ACT,\sim,\{\rel{a}\mid a\in \ACT \},V}$ where:
	\begin{itemize}
		\item $W$ is a non-empty set,
        \item \ACT\ is a set of actions,
		\item $\sim\ \subseteq {W\times W}$ is an equivalence relation on $W$,
		\item $\rel{a}\ \subseteq {W\times W}$ is a binary relation on $W$, and 
		\item $V:W\to 2^\PROP$ is a valuation.
	\end{itemize}
\end{definition}
Note that the labels in \ACT\ do not appear in the language. The graph in Example \ref{ex.pain} represents a model with omitted self-loops of $\sim$ (dotted lines), and the equivalence classes induced by $\sim$ are $\{s_1,s_2\}, \{s_3\}, \{s_4\}, \{s_5\}, \{s_6\}$. In this paper we do not require any properties between $\sim$ and $\rel{a}$ to lay out the most general framework. We will come back to particular assumptions like \textit{perfect recall} at the end of the paper. Given a model and a state $s$, if there exists $t$ such that $s\rel{a}t$, we say that $a$ is \emph{executable} at $s$. Also note that the actions can be non-deterministic. For each $s\in W$, we use $[s]$ to denote the equivalence class $\{t\in W\mid s\sim t \}$, and use $\EQ{W}$ to denote the collection of all the equivalence classes on $W$ w.r.t.\ $\sim$. We use $[s]\rel{a}[t]$ to indicate that there are $s'\in[s]$ and $t'\in[t]$ such that $s'\rel{a}t'$. If there is $t\in W$ such that $[s]\rel{a}[t]$, we say $a$ is executable at $[s]$. 
\begin{definition}[Strategies] Given a model, a (uniformly executable) strategy is a partial function $\sigma:[W]\to \ACT$ such that $\sigma([s])$ is executable at all $s'\in[s]$. Particularly, the empty function is also a strategy, the \emph{empty strategy}.
\end{definition}
Note that the executability is as crucial as uniformity, without which the knowledge-how may be trivialized. We use $\Dom(\sigma)$ to denote the domain of $\sigma$. Function $\sigma$ can be seen as a binary relation such that $([s],a),([s],b)\in \sigma$ implies $a=b$. Therefore, if $\tau$ is a restriction of $\sigma$, i.e.\ $\tau\subseteq \sigma$, it follows that $\Dom(\tau)\subseteq \Dom(\sigma)$, and $\tau([s])=\sigma([s])$ for all $[s]\in\Dom(\tau)$. 

\begin{definition}[Executions]
	 Given a strategy $\sigma$ w.r.t a model $\M$, a \emph{possible execution} of $\sigma$ is a possibly infinite sequence of equivalence classes $\delta=[s_0][s_1]\cdots$ such that $[s_{i}]\rel{\sigma([s_{i}])}[s_{i+1}]$ for all $0\leq i< |\delta|$.
	 Particularly, $[s]$ is a possible execution if $[s]\not\in\Dom(\sigma)$. If the execution is a finite sequence $[s_0]\cdots[s_n]$, we call $[s_n]$ the \LeafNode, and $[s_i](0\leq i< n)$ an \InnerNode\ w.r.t. this execution. If it is infinite, then all $[s_i] (i\in\mathbb{N})$ are \InnerNode s. A possible execution  of $\sigma$ is \emph{complete} if it is infinite or its \LeafNode\ is not in $\Dom(\sigma)$. 
\end{definition}
Given $\delta=[s_0]\cdots [s_n]$ and $\mu=[t_0]\cdots[t_m]$, we use $\delta\sqsubseteq \mu$ to denote that $\mu$ \textit{extends} $\delta$, i.e.,  $n\leq m$ and $[s_i]=[t_i]$ for all $0\leq i\leq n$. If $\delta\sqsubseteq \mu$, we define $\delta\sqcup\mu=\mu$.
We use $\CELeafN(\sigma,s)$ to denote the set of all  \LeafNode s of all the $\sigma$'s complete executions (can be many due to non-determinism) starting from $[s]$, and $\CEInnerN(\sigma,s)$ to denote the set of all the \InnerNode s of $\sigma$'s complete executions starting from $[s]$. $\CELeafN(\sigma,s)\cap\CEInnerN(\sigma,s)=\emptyset$ since if $[s]$ is a \LeafNode\ of a complete execution  then $\sigma$ is not defined at $[s]$. 


%


\begin{definition}[Semantics] 
Given a pointed model $\M,s$, the satisfaction relation $\vDash$ is defined as follows:
	\begin{tabular}{LLL}
	\M,s\vDash p &\iff& p\in \V(s)\\
	\M,s\vDash \neg\phi &\iff& \M,s\nvDash \phi\\
	\M,s\vDash \phi\wedge\psi &\iff& \M,s\vDash\phi \text{ and } \M,s\vDash \psi\\
	\M,s\vDash \K{\phi} &\iff& \text{for all }s':s{\sim}s'\text{ implies }\M,s'\vDash\phi\\
	\M,s\vDash \Kh{\phi} &\iff& \text{there exists a strategy } \sigma\text{ such that }\\
    &&1. [t]{\subseteq}\llrr{\phi} \text{ for all }[t]{\in}\CELeafN(\sigma,s)\\
    &&2.\text{ all its complete executions}\\ 
    && \text{ starting from $[s]$ are finite,}
		\\
	\end{tabular}
where $\llrr{\phi}=\{s\in W\mid \M,s\vDash\phi \}$.
\end{definition}
Note that the two conditions for $\sigma$ in the semantics of $\Kh$ reflect our two intuitions mentioned in the introduction. The implicit role of $\K$ in $\Kh$ will become more clear when the axioms are presented. 
Going back to Example \ref{ex.pain}, we can verify that $\Kh \neg p$ holds on $s_1$ and $s_2$ due to the strategy $\sigma=\{\{s_1,s_2\}\mapsto \texttt{test}, \{s_3\}\mapsto \texttt{pills}, \{s_4\}\mapsto \texttt{surgery}\}$. Note that $\CELeafN(\sigma,s)=\{[s_5], [s_6]\}=\{\{s_5\}, \{s_6\}\}$ and $\llrr{\neg p}=\{s_5, s_6\}$. On the other hand, $\Kh \neg q$ is not true on  $s_1$: although the agent can guarantee $\neg q$ \textit{de facto} on $s_1$ by taking a strategy such that $\{s_1,s_2\}\mapsto \texttt{test}$ and $\{s_3\}\mapsto \texttt{pills}$, he cannot know it beforehand since nothing works at $s_2$ to make sure $\neg q$. Readers may also verify that $\Kh(p\lra q)$ holds at $s_1$ and $s_2$ (hint: a strategy is a partial function). 

\section{Axiomatization}
\label{sec.ax}
\subsection{The proof system}
\begin{center}
	\begin{tabular}{lc}
		\multicolumn{2}{c}{System $\SKh$}\\
		{\textbf{Axioms}}&\\
		\TAUT & \text{all axioms of propositional logic}\\
		\DISTK & $\K p\land\K (p\to q)\to \K q$\\
		\AxTrK& $\K p\to p $ \\
		\AxTransK& $\K p\to\K\K p$\\
		\AxEucK& $\neg \K p\to\K\neg\K p$\\
		\AxKtoKh &$\K p \to \Kh p$ \\	
		\AxKhtoKhK&$\Kh p \to \Kh\K p$  \\
		\AxKhtoKKh&$\Kh p \to \K\Kh p$  \\	
		\AxKhKh&$\Kh \Kh p \to \Kh p$  \\
		\AxKhbot&$\Kh \bot \to \bot$  \\
\end{tabular}
	\begin{tabular}{lclc}
\textbf{Rules:}\\
 \MP & $\dfrac{\varphi,\varphi\to\psi}{\psi}$&\NECK &$\dfrac{\varphi}{\K\varphi}$\\
 \EQREPKh& $\dfrac{\varphi\to\psi}{\Kh\varphi\to\Kh\psi}$ & \SUB & $\dfrac{\varphi(p)}{\varphi[\psi\slash p]}$
\end{tabular}
\end{center}
Note that we have $\mathbb{S}5$ axioms for $\K$. $\AxKtoKh$ says if $p$ is known then you know how to achieve $p$ by doing nothing (we allow the empty strategy). $\AxKhtoKhK$ reflects the first condition in the semantics that the goal is known after the executions. We will come back to this axiom at the end of the paper. Note that the termination condition is not fully expressible in our language but \AxKhbot\ captures part of it by ruling out strategies that have no terminating executions at all. \AxKhKh\ essentially says that the strategies can be composed. Its validity is quite involved, to which we devote the next subsection.  Finally, $\AxKhtoKKh$ is the positive introspection axiom for $\Kh$, whose validity is due to uniformity of the strategies on indistinguishable states. The corresponding negative introspection can be derived by using \AxKhtoKKh, \AxEucK\ and \AxTrK: 
\begin{proposition}
	$\vdash\neg\Kh p \to \K\neg\Kh p$.
\end{proposition}
\begin{proof}
~\\
$\begin{array}{lll}
	(1) & \neg\K\Kh p\to \neg\Kh p &{\AxKhtoKKh}\\
	(2) & \K\neg\K\Kh p\to \K\neg\Kh p &{(1),\NECK}\\
	(3) & \neg\K\Kh p\to \K\neg\K\Kh p & \AxEucK\\
	(4) & \neg\K\Kh p\to \K\neg\Kh p & (2),(3),\MP\\
	(5) & \neg\Kh p\to \neg \K\Kh p & \AxTrK\\
	(6) & \neg\Kh p\to \K\neg\Kh p & (4),(5),\MP\\
\end{array}$
\end{proof}
Note that we do not have the K axiom for $\Kh$. Instead, we have the monotonicity rule $\EQREPKh$. In fact, the logic is not normal, as desired, e.g., $\Kh p\land\Kh q\to \Kh(p\land q)$ is not valid: the existence of two different strategies for different goals does not imply the existence of a unified strategy to realize both goals. 
\subsection{Validity of $\AxKhKh$}
$\AxKhKh$ is about the ``sequential'' compositionality of strategies. Suppose on some pointed model there is a strategy $\sigma$ to guarantee that we end up with the states where on each $s$ of them we have some other strategy $\sigma_s$ to make sure $p$ ($\Kh\Kh p$). Since the strategies are uniform, we only need to consider some $\sigma_{[s]}$ for each $[s]$. Now to validate $\AxKhKh$, we need to design a unified strategy to compose $\sigma$ and those $\sigma_{[s]}$ into one strategy to still guarantee $p$ ($\Kh p$). The general idea is actually simple: first ordering those leafnodes $[s]$ (using Axiom of Choice); then by transfinite induction  adjust $\sigma_{[s]}$ one by one to make sure these strategies can fit together as a unified strategy $\theta$; finally, merge the relevant part of $\sigma$ with $\theta$ into the desired strategy. We make this idea precise below. 
First we need an observation: 
\begin{proposition}\label{prop:subfuction}
	Given strategies $\tau$ and $\sigma$ with $\tau\subseteq \sigma$, if $[s]\in\Dom(\tau)$ and $\Dom(\sigma)\cap\CELeafN(\tau,s) =\emptyset $, then a sequence is $\sigma$'s complete execution from $[s]$ if and only if it is $\tau$'s complete execution from $[s]$.
\end{proposition}
\begin{proof}
	Left to Right: Let $[s_0]\cdots [s_n]\cdots$ be a $\sigma$'s complete execution from $[s]$. We will show it is also a $\tau$'s complete execution from $[s]$. Firstly, we show it is a possible execution give $\tau$ from $[s]$. If it is not, there exists $[s_i]$ such that $[s_i]$ is not the \LeafNode\ of this execution and that $[s_i]\not\in\Dom(\tau)$. Let $[s_j]$ be the minimal equivalence class in the sequence with such properties. It follows that $[s_j]\in\CELeafN(\tau,s)$ and $[s_j]\in\Dom(\sigma)$. These are contradictory with $\Dom(\sigma)\cap\CELeafN(\tau,s) =\emptyset $. 
	
	Next we will show that $[s_0]\cdots [s_n]\cdots$ be a $\tau$'s complete execution from $[s]$. It is obvious if the sequence is infinite. If it is finite, let the \LeafNode\ is $[s_m]$. It follows that $[s_m]\not\in\Dom(\sigma)$. Since $\tau\subseteq \sigma$, it follows $[s_m]\not\in\Dom(\tau)$. Therefore, the execution is complete given $\tau$.
	
	Right to Left: Let $[s_0]\cdots [s_n]\cdots$ be a $\tau$'s complete execution from $[s]$, we will show it is also a $\sigma$'s complete execution from $[s]$. Since $\tau\subseteq \sigma$, it is also a possible execution given $\sigma$. If the execution is infinite, it is obvious. If it is finite, let the \LeafNode\ is $[s_m]$. It follows that $[s_m]\in\CELeafN(\tau,s)$. Since $\Dom(\sigma)\cap\CELeafN(\tau,s)=\emptyset$, it follows that $[s_m]\not\in\Dom(\sigma)$. Therefore, the execution is also complete give $\sigma$.
\end{proof}

%

\begin{proposition}
	$\vDash\Kh\Kh\phi\to\Kh\phi$.
\end{proposition}
\begin{proof}
	Supposing $\M,s\vDash\Kh\Kh\phi$, we will show that $\M,s\vDash\Kh\phi$. It follows by the semantics that there exists a strategy $\sigma$ such that all $\sigma$'s complete executions from $[s]$ are finite and ${[t]}\subseteq \llrr{\Kh\phi}$ for all $[t]\in\CELeafN(\sigma,s)\ (\ast)$. If $[s]\not\in\Dom(\sigma)$, then $\CELeafN(\sigma,s)=\{[s] \}$, and then it is trivial that $\M,s\vDash\Kh\phi$. Next we focus on the case of $[s]\in\Dom(\sigma)$. 
	
	According to well-ordering theorem (equivalent to Axiom of Choice), we assume $\CELeafN(\sigma,s)=\{S_i\mid i< \gamma \}$ where $\gamma$ is an ordinal number and $\gamma\geq 1$. Let $s_i$ be an element in $S_i$; then $[s_i]=S_i$. Since  $\M,s_i\vDash\Kh\phi$ for each $i<\gamma$, it follows that for each $[s_i]$ there exists a strategy $\sigma_i$ such that all $\sigma_i$'s complete executions from $[s_i]$ are finite and $[v]\subseteq \llrr{\phi}$ for all $[v]\in\CELeafN(\sigma_i,s_i)\ (\blacktriangleleft)$. Next, in order to show $\M,s\vDash\Kh\phi$, we need to define a strategy $\tau$. The definition consists of the following steps. 

\medskip
	\noindent\textbf{Step I}. By induction on $i$, we will define a set of strategies $\tau_i$ where $0\leq i< \gamma$. Let $f_i=\bigcup_{\beta< i}\tau_\beta$ and $D_i=\CEInnerN(\sigma_{i},s_{i})\setminus(\Dom(f_i)\cup \{[v]\in\CELeafN(f_i,t)\mid [t]\in\Dom(f_i) \})$ we define:
	\begin{itemize}
		\item $\tau_0=\sigma_0|_{\CEInnerN(\sigma_0,s_0)}$; 
         \item $\tau_i=f_i\cup (\sigma_i|_{D_i})$ for $i>0$.
	\end{itemize}
\begin{claim}\label{claim:tau_i}
We have the following results:
\begin{itemize}
\item[1.]For each $0\leq i< \gamma$, $\tau_j\subseteq\tau_i$ if $j< i$;
\item[2.]For each $0\leq i< \gamma$, $\tau_i$ is a partial function;
\item[3.]For each $0\leq i< \gamma$, $\Dom(\tau_i)\cap \CELeafN(\tau_j,t)=\emptyset $ where $t\in\Dom(\tau_j)$ if $j< i$;
\item[4.]For each $0\leq i< \gamma$, if $\delta=[t_0]\cdots$ is a $\tau_i$'s complete execution from $[t]\in\Dom(\tau_i)$ then $|\delta|=n$ for some $n\in\mathbb{N}$ and $[t_n]\subseteq\llrr{\phi}$;
\item[5.]For each $0\leq i< \gamma$, $[s_i]\in\Dom(\tau_i)$ or $[s_i]\subseteq\llrr{\phi}$.
\end{itemize}
\end{claim}    
\begin{claimproof}
\begin{itemize}
\item[1.] It is obvious.
\item[2.] We prove it by induction on $i$. For the case of $i=0$, it is obvious. For the case of $i=\alpha>0$, it follows by the IH that $\tau_\beta$ is a partial function for each $\beta<\alpha$. Furthermore, it follows by 1. that $\tau_{\beta_1}\subseteq\tau_{\beta_2}$ for all $\beta_1<\beta_2<\alpha$. Thus, we have $f_\alpha=\bigcup_{\beta<\alpha}\tau_\beta$ is a partial function. Since $\sigma_\alpha$ is a partial function, in order to show $\tau_\alpha$ is a partial function, we only need to show that $\Dom(f_\alpha)\cap D_\alpha=\emptyset $. Since $D_\alpha=\CEInnerN(\sigma_{\alpha},s_{\alpha})\setminus\Dom(f_\alpha)\setminus\{[v]\in\CELeafN(f_\alpha,t)\mid t\in\Dom(f_\alpha) \}$, it is obvious that $\Dom(f_\alpha)\cap D_\alpha=\emptyset $.
\item[3.] We prove it by induction on $i$. It is obvious for the case of $i=0$. For the case of $i=\alpha>0$, given $j<\alpha$ and $t\in\Dom(\tau_j)$, we need to show that $\Dom(\tau_\alpha)\cap\CELeafN(\tau_j,t)=\emptyset$. Supposing $[v]\in\CELeafN(\tau_j,t)$, we will show that $[v]\not\in\Dom(\tau_\alpha)$, namely $[v]\not\in \Dom(f_\alpha)\cup D_\alpha$. Since $j<\alpha$ and $f_\alpha=\bigcup_{\beta<\alpha}\tau_\alpha$, it follows $t\in\Dom(f_\alpha)$. Moreover, due to $D_\alpha=\CEInnerN(\sigma_{\alpha},s_{\alpha})\setminus\Dom(f_\alpha)\setminus\{[v]\in\CELeafN(f_\alpha,t)\mid t\in\Dom(f_\alpha) \}$, it follows $[v]\not\in D_\alpha$. 

Next, we only need to show $[v]\not\in \Dom(f_\alpha)$. Assuming $[v]\in\Dom(f_\alpha)$, it follows that $[v]\in\Dom(\tau_\beta)$ for some $\beta<\alpha$. There are two cases: $j<\beta$ or $j\geq\beta$. If $j<\beta$, it follows by the IH that $\Dom(\tau_\beta)\cap\CELeafN(\tau_j,t)=\emptyset$. Contradiction. If $j\geq \beta$, it follows by 1. that $\tau_\beta\subseteq\tau_j$. Due to $[v]\in\Dom(\tau_\beta)$, it follows $[v]\in\Dom(\tau_j)$. It is contradictory with $[v]\in\CELeafN(\tau_j,t)$. Thus, we have $[v]\not\in\Dom(f_\alpha)$.

\item[4.] We prove it by induction on $i$. For the case of $i=0$, due to $\Dom(\tau_0)=\CEInnerN(\sigma_0,s_0)$, it follows that there is a $\sigma_0$'s possible execution $[s_0]\cdots [s_m]$ such that $m\in\mathbb{N}$ and $[s_m]=[t]$. Let $\mu=[s_0]\cdots[s_{m-1}]\circ\delta$. (If $m=0$ then $\mu=\delta$). Since $\delta$ is a $\tau_0$'s complete execution from $[t]$, it follows that $\mu$ is a $\sigma_0$'s complete execution from $[s_0]$. It follows by ($\blacktriangleleft$) that $\mu$ is finite. Thus, $\delta=[t_0]\cdots[t_n]$ for some $n\in\mathbb{N}$. Since $[t_n]\in\CELeafN(\sigma_0,s_0)$, it follows by $(\blacktriangleleft)$ that $[t_n]\subseteq\llrr{\phi}$.

For the case of $i=\alpha>0$, there are two situations: $[t]\in\Dom(f_\alpha)$ or $[t]\in D_\alpha$.
		 If $[t]\in\Dom(f_\alpha)$, it follows that $[t]\in\Dom(\tau_\beta)$ for some $\beta<\alpha$. By 3, we have $\Dom(\tau_\alpha)\cap\CELeafN(\tau_\beta,t)=\emptyset$. Since $\delta$ is a $\tau_\alpha$'s complete execution, it follows by Proposition \ref{prop:subfuction} that $\delta$ is also a $\tau_\beta$'s complete execution from $[t]$. It follows by the IH that $|\delta|=n$ for some $n\in\mathbb{N}$ and $[t_n]\subseteq\llrr{\phi}$.
         		 
		 If $[t]\in D_\alpha$, 
         there are two cases: there exist $k<|\delta|$ and $\beta<\alpha$ s.t. $[t_k]\in \Dom(\tau_\beta)$, or there do not exist such $k$ and $\beta$. (Please note that $|\delta|>1$ due to the fact that $\delta=[t_0]\cdots$ is $\tau_\alpha$'s complete execution from $[t]\in\Dom(\tau_\alpha)$).
		 \begin{itemize}
		 	\item $[t_k]\in\Dom(\tau_\beta)$ for some $k<|\delta|$ and some $\beta<\alpha$: It follows that $\mu=[t_k]\cdots$ is a $\tau_{\alpha}$'s complete execution from $[t_k]$. By 3. and Proposition \ref{prop:subfuction}, $\mu$ is a $\tau_\beta$'s complete execution from $[t_k]$. By IH, $\mu=[t_k]\cdots[t_{k+n}]$ for some $n\in\mathbb{N}$ and $[t_{k+n}]\subseteq\llrr{\phi} $. Therefore, $|\delta|=k+n$.
		 	\item If there do not exist $k<|\delta|$ and $\beta<\alpha$ s.t. $[t_k]\in\Dom(\tau_\beta)$, it follows that $\delta=[t_0]\cdots$ is a $\sigma_{\alpha}$'s possible execution from $[t]$. 
            Since $[t]\in D_\alpha\subseteq\CEInnerN(\sigma_{\alpha},s_{\alpha})$, then there is a $\sigma_\alpha$'s possible execution $[s_0]\cdots[s_m]$ s.t. $m\in\mathbb{N}$, $[s_0]=[s_\alpha]$ and $[s_m]=[t] $. Let $\mu=[s_0]\cdots[s_{m-1}]\circ\delta$. (If $m=0$ then $\mu=\delta$). It follows that $\mu$ is $\sigma_\alpha$'s possible execution from $s_\alpha$. By $(\blacktriangleleft)$, all $\sigma_\alpha$'s complete executions from $s_\alpha$ are finite. Thus, $\mu$ is finite. Therefore, $\delta=[t_0]\cdots[t_n]$ for some $n\in\mathbb{N}$.
            
            We continue to show that $[t_n]\subseteq\llrr{\phi}$. Since $\delta=[t_0]\cdots[t_n]$ is a $\tau_\alpha$'s complete execution from $t$ and it is also a $\sigma_\alpha$'s possible execution from $t$, there are two cases: $[t_n]\in \CELeafN(f_\alpha,t')$ for some $t'\in \Dom(f_\alpha)$, or $\delta$ is a $\sigma_\alpha$'s complete execution from $t$. If $[t_n]\in \CELeafN(f_\alpha,t')$ for some $t'\in \Dom(f_\alpha)$, then there exists $\beta<\alpha$ s.t. $[t]\in\CELeafN(\tau_\beta,t')$ and $[t']\in\Dom(\beta)$. By IH, $[t_n]\subseteq \llrr{\phi}$. If $\delta$ is a $\sigma_\alpha$'s complete execution from $t$, it follows that $\mu$ is a $\sigma_\alpha$'s complete execution from $[s_\alpha]$. Then by $(\blacktriangleleft)$, we have $[t_n]\subseteq\llrr{\phi}$.
            
 %
		 \end{itemize}
 
 \item[5.] If $[s_i]\not\in\Dom(\sigma_i)$, it follows by $(\blacktriangleleft)$ that $[s_i]\subseteq\llrr{\phi}$. Otherwise, there are two cases: $i=0$ or $i=\alpha>0$. If $i=0$, it follows by $[s_0]\in\Dom(\sigma_0)$ that $[s_0]\in\CEInnerN(\sigma_0,s_0)$. Thus, $[s_0]\in\Dom(\tau_0)$.
 
 If $i=\alpha>0$ and $[s_\alpha]\in \Dom(\sigma_\alpha)$, we will show that if $[s_\alpha]\not\in\Dom(\tau_\alpha)$ then $[s_\alpha]\subseteq\llrr{\phi}$. Firstly, we have that $[s_i]\in\CEInnerN(\sigma_\alpha,s_\alpha)$. Since $[s_\alpha]\not\in\Dom(\tau_\alpha)$, it follows that $[s_\alpha]\in\CELeafN(f_\alpha,t)$ for some $[t]\in\Dom(f_\alpha)$. It follows that there exists $\beta<\alpha$ such that $[s_\alpha]\in\CELeafN(\tau_\beta,t)$ and $t\in\Dom(\tau_\beta)$. It follows by 4. that $[s_i]\subseteq\llrr{\phi}$.
\end{itemize}
\end{claimproof}

\medskip
\noindent\textbf{Step II}. We define $\tau_\gamma=\bigcup_{i<\gamma}\tau_i$. It follows by 1.\ and 2.\ of Claim \ref{claim:tau_i} that $\tau_\gamma$ is indeed a partial function. Then we prove the following claim.
	\begin{claim}\label{claim:tau'}
If $\delta=[t_0]\cdots$ is a $\tau_\gamma$'s complete execution from $[t]\in\Dom(\tau_\gamma)$ then $|\delta|=n$ for some $n\in\mathbb{N}$ and $[t_n]\subseteq\llrr{\phi}$
	\end{claim}
    \begin{claimproof}
Since $[t]\in\Dom(\tau_\gamma)$, it follows that $[t]\in\Dom(\tau_i)$ for some $i<\gamma$. It follows by 5. of Claim \ref{claim:tau_i} that all $\tau_i$'s complete executions from $[t]$ are finite. Thus, there exists $\mu\sqsubseteq\delta$ such that $|\mu|=n$ for some $n\in\mathbb{N}$ and $\mu$ is $\tau_i$'s complete execution from $[t]$. It follows by 5. of Claim \ref{claim:tau_i} that $[t_n]\subseteq\llrr{\phi}$.

Next, we only need to show $\delta=\mu$. If not, then $\delta=[t_0]\cdots[t_n][t_{n+1}]\cdots$. We then have that there exists $j<\gamma$ such that $\{t_k\mid 0\leq k\leq n\}\subseteq \Dom(\tau_j)$. It cannot be that $j\leq i $. Otherwise, $\mu$ is not $\tau_i$'s complete execution since $\tau_j\subseteq\tau_i$ by 1. of Claim \ref{claim:tau_i}. Thus, we have $j>i$. Since we also have that $[t_n]\in\Dom(\tau_j)$, $[t_n]\in\CELeafN(\tau_i,t)$ and $t\in\Dom(\tau_i)$, this is contradictory with 3. of Claim \ref{claim:tau_i}. Therefore, we have $\delta=\mu$.
\end{claimproof}

\medskip
\noindent\textbf{Step III}. We define $\tau$ as $\tau=\tau_\gamma\cup(\sigma|_C)$ where $C=\CEInnerN(\sigma,s)\setminus(\Dom(\tau_\gamma)\cup \{[v]\in \CELeafN(\tau',t)\mid [t]\in\Dom(\tau_\gamma)\})$ and $\sigma$ is the strategy mentioned at $(\ast)$.
Since both $\tau_\gamma $ and $\sigma|_C$ are partial functions, $\tau$ is also a partial function.
We then prove the following claim.
	\begin{claim}\label{claim:tau}
If $\delta=[t_0]\cdots$ is a $\tau$'s complete execution from $[t]\in\Dom(\tau)$ then $|\delta|=n$ for some $n\in\mathbb{N}$ and $[t_n]\subseteq\llrr{\phi}$.
	\end{claim}
	\begin{claimproof}
		Since $\Dom(\tau)=\Dom(\tau_\gamma)\cup C$, there are two cases: $[t]\in\Dom(\tau_\gamma)$ or $[t]\in C$. 
		
		If $[t]\in\Dom(\tau_\gamma)$, it follows that $\CELeafN(\tau_\gamma,t)\cap C=\emptyset$. Moreover, we have  $\CELeafN(\tau_\gamma,t)\cap\Dom(\tau_\gamma)=\emptyset$. Thus, we have $\CELeafN(\tau_\gamma,t)\cap\Dom(\tau)=\emptyset$. It follows by Proposition \ref{prop:subfuction} that $\delta$ is a $\tau_\gamma$'s complete execution from from $[t]$. It follows by Claim \ref{claim:tau'} $|\delta|=n$ for some $n\in\mathbb{N}$ and $[t_n]\subseteq\llrr{\phi}$
		
		If $[t]\in C$, there are two cases: there exists $k<|\delta|$ such that $[t_k]\in \Dom(\tau_\gamma)$, or there does not exists such $k$. (Please note that $|\delta|>1$ due to the fact that $\delta=[t_0]\cdots$ is $\tau$'s complete execution from $[t]\in\Dom(\tau)$).
		\begin{itemize}
			\item $[t_k]\in\Dom(\tau_\gamma)$ for some $k<|\delta|$: It follows that $\mu=[t_k]\cdots$ is a $\tau$'s complete execution from $[t_k]$. Since $\Dom(\tau)\cap\CELeafN(\tau_\gamma,t_k)=\emptyset$. It follows by Proposition \ref{prop:subfuction} that $\mu$ is a $\tau_\gamma$'s complete execution from $[t_k]$. It follows by Claim \ref{claim:tau'} that $\mu=[t_k]\cdots[t_{k+n}]$ for some $n\in\mathbb{N}$ and $[t_{k+n}]\subseteq\llrr{\phi} $. Therefore, $|\delta|=k+n$.
            		 	\item If there does not exist $k<|\delta|$ s.t. $[t_k]\in\Dom(\tau_\gamma)$, then $\delta=[t_0]\cdots$ is a $\sigma$'s possible execution from $[t]$. 
            Since $[t]\in C\subseteq\CEInnerN(\sigma,s)$, then there is a $\sigma$'s possible execution $[s_0]\cdots[s_m]$ s.t. $m\in\mathbb{N}$, $[s_0]=[s]$ and $[s_m]=[t] $. Let $\mu=[s_0]\cdots[s_{m-1}]\circ\delta$. (If $m=0$ then $\mu=\delta$). It follows that $\mu$ is $\sigma$'s possible execution from $s$. By $(\ast)$, all $\sigma$'s complete executions from $s$ are finite. Thus, $\mu$ is finite. Therefore, $\delta=[t_0]\cdots[t_n]$ for some $n\in\mathbb{N}$.
            
            We continue to show that $[t_n]\subseteq\llrr{\phi}$. Since $\delta=[t_0]\cdots[t_n]$ is a $\tau$'s complete execution from $t$ and it is also a $\sigma$'s possible execution from $t$, there are two cases: $[t_n]\in \CELeafN(\tau_\gamma,t')$ for some $t'\in \Dom(\tau_\gamma)$, or $\delta$ is a $\sigma$'s complete execution from $t$. If $[t_n]\in \CELeafN(\tau_\gamma,t')$ for some $[t']\in \Dom(\tau_\gamma)$, it follows  by Claim \ref{claim:tau'} that $[t_n]\subseteq \llrr{\phi}$. If $\delta$ is a $\sigma$'s complete execution from $t$, it follows that $\mu$ is a $\sigma$'s complete execution from $[s]$. It follows that $[t_n]=S_i$ for some $0\leq i<\gamma$. Since $\delta=[t_0]\cdots[t_n]$ is $\tau$'s complete execution from $[t]\in\Dom(\tau_\gamma)$, it follows $[t_n]\not\in\Dom(\tau_\gamma)$. We then have $[t_n]\not\in\Dom(\tau_i)$, namely $S_i\not\in\tau_i$. It follows by 5. of Claim \ref{claim:tau_i} that $S_i\subseteq\llrr{\phi}$, namely $[t_n]\subseteq\llrr{\phi}$.
\end{itemize} \end{claimproof}

Next, we continue to show that $\calM,s\vDash\Kh\phi$ with the assumption that $[s]\in\Dom(\sigma)$.	
	Since $[s]\in\Dom(\sigma)$, we have $[s]\in\CEInnerN(\sigma,s)$. There are two cases: $[s]\in\Dom(\tau)$ or not. If $[s]\in\Dom(\tau)$, it follows by claim \ref{claim:tau} that $\M,s\vDash\Kh\phi$. If $[s]\not\in\Dom(\tau)$, due to $[s]\in\CEInnerN(\sigma,s)$, it follows that $[s]\in\CELeafN(\tau_\gamma,t)$ for some $[t]\in\Dom(\tau_\gamma)$. It follows by Claim \ref{claim:tau'} that $[s]\subseteq\llrr{\phi}$. It follows that $\calM,s\vDash \K\phi$. It is obvious that $\calM,s\vDash\Kh\phi$.
\end{proof}

\begin{theorem}[Soundness]\label{them:soundness}
	If $\vdash\phi$ then $\vDash\phi$.
\end{theorem}

\section{Completeness and Decidability}\label{sec.proof}
Let $\SFC$ be a subformula-closed set of formulas. It is obvious that $\SFC$ is countable since the whole language itself is countable. Given a set of formulas $\Delta$, let:
$\Delta|_{\K}=\{\K\phi\mid\K\phi\in\Delta \}$, 
$\Delta|_{\neg\K}=\{\neg\K\phi\mid\neg\K\phi\in\Delta \}$, 
$\Delta|_{\Kh}=\{\Kh\phi\mid\Kh\phi\in\Delta \}$, 
$\Delta|_{\neg\Kh}=\{\neg\Kh\phi\mid\neg\Kh\phi\in\Delta \}$.
Below we define the closure of $\Phi$, and use it to build a canonical model w.r.t.\ $\Phi$. We will show that when $\Phi$ is finite then we can build a finite model.
\begin{definition}\label{def.closure}
	$\cl{\SFC}$ is defined as:
    \[\cl{\SFC}=\SFC\cup\{\K\phi\mid\phi\in\SFC\}.\]
\end{definition}
\begin{definition}[Atom]
	Let $\cl{\SFC}=\{\psi_i\mid i\in\mathbb{N}\}$. The formula set $\Delta=\{Y_i\mid i\in\mathbb{N}\}$ is an atom of $\cl{\SFC}$ if 
	\begin{itemize}
		\item $Y_i=\psi_i$ or $Y_i=\neg\psi_i$ for all $\psi_i\in \cl{\SFC}$;
		\item $\Delta$ is consistent.
	\end{itemize}
\end{definition}
Note that if $\Phi$ is the whole language then an atom is simply a maximal consistent set. 
By a standard inductive construction, we can obtain the Lindenbaum-like result in our setting (which is useful to show the existence lemma for $\K$):
\begin{proposition}\label{pro.lindenbaumForAtom}
	Let $\Delta$ be an atom of $\cl{\SFC}$, $\Gamma\subseteq\Delta$ and  $\phi\in\cl{\SFC}$. If $\Gamma\cup\{\pm\phi\}$ is consistent then there is an atom $\Delta'$ of $\cl{\SFC}$ such that $(\Gamma\cup\{\pm{\phi}\})\subseteq \Delta'$, where $\pm\phi=\phi$ or $\pm\phi=\neg\phi$.
\end{proposition}
\begin{proof}
	Let $\Gamma=\{\phi_k\mid k\in\mathbb{N}\}$.  
	Since $\Delta$ is an atom and $\Gamma\subseteq \Delta$, it follows that there is a set $\Gamma'=\{\chi_k\in\cl{\SFC}\mid k\in\mathbb{N}\}$ such that $\phi_k=\chi_k$ or $\phi_k=\neg\chi_k$ for all $k\in\mathbb{N}$.
	Let $\psi_1,\cdots,\psi_n,\cdots$ be all the formulas in $\cl{\SFC}\setminus\Gamma'\setminus\{\phi \}$. 
	We define $\Gamma_i$ as below.
	\begin{align*}
	\Gamma_0&=\Gamma\cup\{\pm\phi\}\\
	\Gamma_{i+1}&=\begin{cases}
	\Gamma_i\cup\{\psi_i\} & \text{ if } \Gamma_i\cup\{\psi_i\} \text{ is consistent}\\
	\Gamma_i\cup\{\neg\psi_i\} & \text{else}
	\end{cases}
	\end{align*}
	
	Firstly, we will show $\Gamma_i$ is consistent for all $i\in\mathbb{N}$. Since $\Gamma_0$ is consistent, we only need to show that if $\Gamma_i$ is consistent then $\Gamma_{i+1}$ is consistent, i.e.\ either $\Gamma_i\cup\{\psi_i\}$ or $\Gamma_i\cup\{\neg\psi_i\}$ is consistent. Assuming both $\Gamma_i\cup\{\psi_i\}$ and $\Gamma_i\cup\{\neg\psi_i\}$ are not consistent, it follows that $\Gamma_i\vdash\neg\psi_i$ and $\Gamma_i\vdash\psi_i$. That is, $\Gamma_i$ is inconsistent. Contradiction. Therefore, either $\Gamma_i\cup\{\psi_i\}$ or $\Gamma_i\cup\{\neg\psi_i\}$ is consistent.
	
	Let $\Delta'=\bigcup_{i\in\mathbb{N}}\Gamma_i$. It follows that $\Delta'$ is consistent. It is obvious that either $\psi\in\Delta'$ or $\neg\psi\in\Delta'$ for all $\psi\in\cl{\SFC}$. Therefore, $\Delta'$ is an atom of $\cl{\SFC}$.
\end{proof}

\begin{definition}
Given a subformula-closed $\SFC$, the canonical model $\M^{\SFC}=\lr{W,\ACT,\sim,\{\rel{x}\mid x\in\ACT\},V}$ is defined as:
	\begin{itemize}
		\item $W=\{\Delta\mid\Delta$ is an atom of $\cl{\SFC}\}$;
		\item $\ACT=\{\phi\mid\Kh\phi\in\SFC\}$;
		\item $\Delta\sim\Delta'$ iff $\Delta|_{\K}=\Delta'|_{\K}$;
		\item for each $\phi\in\ACT$, $\Delta\rel{\phi}\Delta'$ iff $\Kh\phi,\neg\K\phi\in\Delta$ and $\K\phi\in\Delta'$;
		\item for each $p\in\SFC$, $p\in V(\Delta)$ iff $p\in\Delta$.
	\end{itemize}
\end{definition}
\noindent Note that we use formulas that the agent knows how to achieve as the action labels, and we introduce an action transition if it is necessary, i.e., $\Kh\phi$ but $\neg \K\phi$ (empty strategy does not work). Requiring $\K\phi\in \Delta'$ is to reflect the first condition in the semantics of $\Kh$. Using \NECK, \DISTK\ and Proposition \ref{pro.lindenbaumForAtom}, it is routine to show the existence lemma for $\K$:
\begin{proposition}\label{pro.exeistenceLemmaForK}
	Let $\Delta$ be a state in $\M^{\SFC}$, and $\K\phi\in\cl{\SFC}$. If $\K\phi\not\in \Delta$ then there exists $\Delta'\in[\Delta]$ such that $\neg\phi\in\Delta'$.
\end{proposition}
\begin{proof}
	Let $\Gamma=\Delta|_{\K}\cup\Delta|_{\neg\K}\cup\{\neg\phi\}$. $\Gamma$ is consistent. If not, there are $\K\phi_i,\cdots,\K\phi_n$ and $\neg\K\psi_1,\cdots ,\neg\K\psi_m$ in $\Delta$ such that $$\vdash\K\phi_1\land\cdots\land\K\phi_n\land\neg\K\psi_1\land\cdots \land \neg\K\psi_m\to\phi.$$
	Following by \NECK\ and \DISTK, we have 
	$$\vdash \K(\K\phi_i\land\cdots\land\K\phi_n\land\neg\K\psi_1\land\cdots \land \neg\K\psi_m)\to\K\phi.$$
	Since the epistemic operator $\K$ is distributive over $\land$ and $\vdash \K\K\phi_i\lra\K\phi_i$ for all $1\leq i\leq n$ and $\vdash \K\neg\K\psi_i\lra\neg\K\psi_i$ for all $1\leq i\leq m$, we have 
	$$\vdash\K\phi_i\land\cdots\land\K\phi_n\land\neg\K\psi_1\land\cdots \land \neg\K\psi_m\to\K\phi.$$
	Since $\K\phi_i,\cdots,\K\phi_n$ and $\neg\K\psi_1,\cdots ,\neg\K\psi_m$ are all in $\Delta$ and $\K\phi\in\cl{\SFC}$, it follows that $\K\phi\in\Delta$. It is contradictory with the assumption that $\K\phi\not\in\Delta$. Therefore, $\Gamma$ is consistent. It follows by Proposition \ref{pro.lindenbaumForAtom} that there exists an atom $\Delta'$ of $\cl{\SFC}$ such that $\Gamma\subseteq \Delta'$. Since $(\Delta|_{\K}\cup\Delta|_{\neg\K})\subseteq\Delta'$, we have $\Delta'\sim\Delta$, that is, $\Delta'\in[\Delta]$.
\end{proof}
\begin{proposition}\label{prop:AtomShareTheSameKh}
	Let $\Delta$ and $\Delta'$ be two states in $\M^{\SFC}$ such that $\Delta\sim\Delta'$. We have $\Delta|_{\Kh}=\Delta'|_{\Kh}$.
\end{proposition}
\begin{proof}
	For each $\Kh\phi\in\Delta$, by Definition \ref{def.closure}, $\Kh\phi\in\SFC$. Then $\K\Kh\phi\in\cl{\SFC}$.
	For each $\Kh\phi\in\Delta$, by Axiom \AxKhtoKKh , we have $\K\Kh\phi\in\Delta$. Since $\Delta\sim\Delta'$, then $\K\Kh\phi\in\Delta'$, and by Axiom \AxTrK, $\Kh\phi\in\Delta'$. Then we showed that $\Kh\phi\in\Delta$ implies $\Kh\phi\in\Delta'$. Similarly we can prove $\Kh\phi\in\Delta'$ implies $\Kh\phi\in\Delta$. Hence, $\Delta|_{\Kh}=\Delta'|_{\Kh}$.
\end{proof}
The following is a crucial observation for the latter proofs. 
\begin{proposition}\label{prop:AtomModelAllSuccessorHaskh}
	Let $\Delta$ be a state in $\M^{\SFC}$ and $\psi\in\ACT$ be executable at $[\Delta]$. If $\Kh\phi\in\Delta'$ for all $\Delta'$ with $[\Delta]\rel{\psi}[\Delta']$ then $\Kh\phi\in\Delta$.
\end{proposition}
\begin{proof}
	First, we show that $\K\psi$ is not consistent with $\neg\Kh \phi$. 
	It is obvious that $\Kh\phi\in\cl{\SFC}$.
	Since $\psi$ is executable at $[\Delta]$, there are atoms $\Gamma_1$ and $\Gamma_2$ s.t. $\Gamma_1\rel{\psi}\Gamma_2$. Then $\K\psi\in\Gamma_2$.  
	Assuming that $\K\psi$ is consistent with $\neg\Kh \phi$, by Proposition \ref{pro.lindenbaumForAtom} there exists an atom $\Gamma$ of $\cl{\SFC}$ s.t. $\{\K\psi,\neg\Kh\phi \}\subseteq \Gamma$. Since $\psi\in\ACT$ is executable at $[\Delta]$, then by definition of $\rel{\psi},\sim$ and Proposition \ref{prop:AtomShareTheSameKh}, $\Kh\psi,\neg\K\psi\in\Delta$. It follows that $\Delta\rel{\psi}\Gamma$, then $[\Delta]\rel{\psi}[\Gamma]$. This is contradictory with the assumption that $\Kh\phi\in\Delta'$ for all $\Delta'$ with $[\Delta]\rel{\psi}[\Delta']$. Then $\K\psi$ is not consistent with $\neg\Kh \phi$. Hence, $\vdash \K\psi\to \Kh\phi$. 
	
	Since $\vdash \K\psi\to \Kh\phi$, it follows by Rule \EQREPKh\ and Axiom \AxKhtoKhK\ that $\vdash\Kh\psi\to \Kh\Kh\phi$. Moreover, it follows by Axiom \AxKhKh\ that $\vdash\Kh\psi\to \Kh\phi$. Since ${\psi}$ is executable at $[\Delta]$, 
	it follows by the definition of $\rel{\psi}$ and Proposition \ref{prop:AtomShareTheSameKh} that $\Kh\psi\in\Delta$. Therefore, we have $\Kh\phi\in\Delta$.
\end{proof}
\begin{lemma}
	For each $\phi\in\cl{\SFC}$,  $\M^{\SFC},\Delta\vDash\phi$ iff $\phi\in\Delta$.
\end{lemma}
\begin{proof}
	We prove it by induction on $\phi$. We only focus on the case of  $\Kh\phi$; the other cases are straightforward, e.g., $\K\phi$ case can  proved based on Proposition \ref{pro.exeistenceLemmaForK}.
Note that if $\Kh\phi\in cl(\Phi)$ then $\phi\in cl(\Phi)$ thus by Definition \ref{def.closure} $\K\phi\in\cl{\SFC}$.
		
		\textbf{Right to Left}: If $\Kh\phi\in\Delta$, we will show $\M^{\SFC},\Delta\vDash\Kh\phi$.  Firstly, there are two cases: $\K\phi\in \Delta$ or $\K\phi\not\in\Delta$. If $\K\phi\in\Delta$, then $\K\phi,\phi\in\Delta'$ for all $\Delta'\in[\Delta]$. Since $\phi\in\SFC$, it follows by IH that $\M^{\SFC},\Delta'\vDash\phi$ for all $\Delta'\in[\Delta]$. Therefore, $\M^{\SFC},\Delta\vDash\K\phi$. It follows by Axiom \AxKtoKh\ and the soundness of $\SKh$ that $\M^{\SFC},\Delta\vDash\Kh\phi$. If $\neg\K\phi\in\Delta$, we first show that $\K\phi$ is consistent. If not, namely $\vdash\K\phi\to\bot$, it follows by Rule \EQREPKh\ that $\vdash \Kh\K\phi\to\Kh\bot$. It follows by Axiom \AxKhbot\ that $\vdash\Kh\K\phi\to\bot$. Since $\Kh\phi\in\Delta$, it follows by Axiom \AxKhtoKhK\ that $\Delta\vdash\bot$, which is contradictory with the fact that $\Delta$ is consistent. Therefore, $\K\phi$ is consistent. 
		
		By Proposition \ref{pro.lindenbaumForAtom} there exists an atom $\Delta'$ s.t. $\K\phi\in\Delta'$. Note that $\phi\in\ACT$. Thus, we have $\Delta\rel{\phi}\Delta'$, then $[\Delta]\rel{\phi}[\Delta']$. Let $[\Delta'']$ be an equivalence class s.t. $[\Delta] \rel{\phi}[\Delta'']$, which indicates $\Gamma\rel{\phi}\Gamma''$ for some $\Gamma\in[\Delta]$ and $\Gamma''\in[\Delta'']$. By definition of $\rel{\phi}$ and $\sim$ we get $\K\phi\in\Theta$ for all $\Theta\in[\Delta'']$. By IH, $\M^{\SFC},\Theta\vDash\phi$ for all $\Theta\in[\Delta'']$, namely $[\Delta'']\subseteq\llrr{\phi}$. 
		Moreover, $\rel{\phi}$ is not a loop on $[\Delta]$ because $\neg\K\phi\in\Delta$. 
		Thus, the partial function $\sigma=\{[\Delta]\mapsto \phi \}$ is a strategy s.t. all its complete executions starting from $[\Delta]$ are finite and $[\Delta'']\subseteq \llrr{\phi}$ for each $[\Delta'']\in\CELeafN(\sigma,\Delta)$. Then, $\M^{\SFC},\Delta\vDash\Kh\phi$.
		
		\textbf{Left to Right}: Suppose $\M^{\SFC},\Delta\vDash\Kh\phi$, we will show $\Kh\phi\in\Delta$. By the semantics, there exists a strategy $\sigma$ s.t. all $\sigma$'s complete executions starting from $[\Delta]$ are finite and $[\Gamma]\subseteq\llrr{\phi}$ for all $[\Gamma]\in\CELeafN(\sigma,\Delta)$. By IH, $\phi\in\Gamma'$ for all $\Gamma'\in[\Gamma]$ and $[\Gamma]\in\CELeafN(\sigma,\Delta)$. By Proposition \ref{pro.exeistenceLemmaForK}, we get $\K\phi\in\Gamma$ for all $[\Gamma]\in\CELeafN(\sigma,\Delta)$. By Axiom \AxKtoKh\ and Proposition \ref{prop:AtomShareTheSameKh}, $\Kh\phi\in\Gamma$ for all $[\Gamma]\in\CELeafN(\sigma,\Delta)$. 
		
		If $[\Delta]\not\in\Dom(\sigma)$, it is obvious that $\Kh\phi\in\Delta$ because $[\Delta]\in\CELeafN(\sigma,\Delta)$. Next, we consider the case of $[\Delta]\in\Dom(\sigma)$, then $[\Delta]\in\CEInnerN(\sigma,\Delta)$.	
		In order to show $\Kh\phi\in\Delta$, we will show a more strong result that $\Kh\phi\in\Delta'$ for all $[\Delta']\in\CEInnerN(\sigma,\Delta)$. Firstly, we show the following claim:
		\begin{claim}\label{claim:infiniteExecution}
			If there exists $[\Delta']\in\CEInnerN(\sigma,\Delta)$ such that $\neg\Kh\phi\in\Delta'$ then there exists an infinite execution of $\sigma$ starting from $[\Delta]$.
		\end{claim}
		\begin{claimproof}
			Let $X$ be the set $\{[\Theta]\in\CEInnerN(\sigma,\Delta)\mid \neg\Kh\phi\in\Theta \}$. It follows that $[\Delta']\in X$ and $X\subseteq\Dom(\sigma)$. We define a binary relation $R$ on $X$ as $R=\{([\Theta],[\Theta'])\mid [\Theta]\rel{\sigma([\Theta])}[\Theta'] \}$.
			
			For each $[\Theta]\in X$, we have that $\sigma([\Theta])$ is executable at $[\Theta]$. Since $\neg\Kh\phi\in\Theta$, by Proposition \ref{prop:AtomModelAllSuccessorHaskh} there exists an atom $\Theta'$ s.t. $[\Theta]\rel{\sigma([\Theta])}[\Theta']$ and $\neg\Kh\phi\in\Theta'$. Since $\Kh\phi\in\Gamma $ for all $[\Gamma]\in\CELeafN(\sigma,\Delta)$ and $[\Theta]\in\CEInnerN(\sigma,\Delta)$,we have $[\Theta']\in\CEInnerN(\sigma,\Delta)$. Then $[\Theta']\in X$. Therefore, $R$ is an entire binary relation on $X$, namely for each $[\Theta]\in X$ there is $[\Theta']\in X$ such that $([\Theta],[\Theta'])\in R$. Then by Axiom of Dependent Choice there exists an infinite sequence $[\Theta_0][\Theta_1]\cdots$ s.t. $([\Theta_n],[\Theta_{n+1}])\in R$ for all $n\in\mathbb{N}$.
			
From the definition of $R$, $[\Theta_0][\Theta_1]\cdots$ is a complete execution of $\sigma$ starting from $[\Theta_0]$. Since $[\Theta_0]\in\CEInnerN(\sigma,\Delta)$ and all complete execution of $\sigma$ from $[\Delta]$ are finite, there is a possible execution $[\Delta_0]\cdots[\Delta_j]$ for some $j\in\mathbb{N}$ s.t. $[\Delta_0]=[\Delta]$ and $[\Delta_j]=[\Theta_0]$. Therefore, $[\Delta_0]\cdots[\Delta_j][\Theta_1]\cdots$ is an infinite complete execution of $\sigma$ from $[\Delta]$.
		\end{claimproof}	

Therefore, we have $\Kh\phi\in\Delta'$ for all $[\Delta']\in\CEInnerN(\sigma,s)$. Otherwise, by claim \ref{claim:infiniteExecution} there is an infinite complete execution given $\sigma$ from $[\Delta]$. This is contradictory with all $\sigma$'s complete execution from $[\Delta]$ are finite, then $\Kh\phi\in\Delta'$ for all $[\Delta']\in\CEInnerN(\sigma,s)$. Since $[\Delta]\in\Dom(\sigma)$, we get $[\Delta]\in\CEInnerN(\sigma,\Delta)$. Then $\Kh\phi\in\Delta$.		
\end{proof}
Now let $\Phi$ be the set of all formulas, then each maximal consistent set $\Delta$ is actually an atom which satisfies all its formulas in $\M^\Phi$, according to the above truth lemma. Completeness then follows immediately. 
\begin{theorem}
$\SKh$ is strongly complete.
\end{theorem}
Note that if $\Phi$ is the set of all subformulas of a given formula $\phi$, then $cl(\Phi)$ is still finite. Due to the soundness of $\SKh$ and Proposition \ref{pro.lindenbaumForAtom}, a satisfiable formula $\phi$ must be consistent thus appearing in some atom, and thus $\phi$ is satisfiable in $\M^\Phi$. It is not hard to see that  $|\M^\Phi|\leq 2^{2|\phi|}$ where $2|\phi|$ is the bound on the size of $cl(\Phi)$. This gives us a small model property of our logic, then decidability follows.
\begin{theorem}
	$\SKh$ is decidable. 
\end{theorem}

\section{Conclusion}
In this paper we propose an epistemic logic of both (goal-directed) knowing how and knowing that, and capture the interaction of the two. 
We also present an proof system $\SKh$ for this logic. We have shown that $\SKh$ is sound and complete w.r.t.\ the semantics, and that $\SKh$ is decidable. We hope the axioms are illuminating towards a better understanding of knowing how. 

Note that we do not impose any special  properties between the interaction of $\rel{a}$ and $\sim$ in the models so far. In the future, it would be interesting to see whether assuming properties of \textit{perfect recall} and\slash or \textit{no learning} (cf.\ e.g., \cite{RAK,WangLi12}) can change the logic or not. 

Our notion of knowing how is relatively strong, particularly evidenced by the axiom $\AxKhtoKhK: \Kh\phi\to\Kh\K \phi$, which is due to the first condition of our semantics for $\Kh$,  inspired by planning with uncertainty. We believe it is reasonable for the scenarios where the agent has perfect recall (or, say, never forgets), which is usually assumed implicitly in the discussions on planning (cf. \cite{YLW15}). However, for a forgetful agent it may not be intuitive anymore, e.g., I know how to get drunk when sober  but I may not know how to get to the state that I know I am drunk, assuming drunk people do not know they are drunk. 
The axiom  $\AxKhKh$ is also interesting in distinguishing different types of knowing how. We have been focusing on the goal-directed knowing how \cite{Gochet13}, but for other types of knowing how such as knowing how to swim, the axiom may not be reasonable anymore, e.g., I know how to let myself to know how to swim (by registering an excellent swimming course) does not mean that I know how to swim right now. We leave the discussion of other types of knowing how in the future. Another obvious next step is to consider  knowing how in multi-agent settings.  

\bibliography{kh}
\bibliographystyle{named}
\end{document}